\documentclass[runningheads]{llncs}
\usepackage{times}
\usepackage{graphicx}
\usepackage{latexsym}
\usepackage{amssymb}
\usepackage{amsmath, mathrsfs, bm}
\usepackage{bbm, dsfont}
\usepackage{nicefrac}
\usepackage{subcaption} %should be replaced by \usepackage{subfig}
\usepackage{xcolor}
\usepackage{algorithm}
\usepackage{algpseudocode}
\usepackage{multirow}
\usepackage{dcolumn}
\usepackage{booktabs}
\usepackage{dsfont}
\usepackage{bbold}
\usepackage{bm}

\newcolumntype{C}[1]{>{\centering\let\newline\\\arraybackslash\hspace{0pt}}m{#1}}
\newcommand{\slime}{{\textsc{s-LIME}}}
\newcommand{\neighbors}{$\hat{\mathcal{X}}$}

\newcommand{\EE}{\mathds{E}} %expectation

\begin{document}

\title{\slime: Reconciling Locality and Fidelity in Linear Explanations}

\author{
Romaric Gaudel\inst{1},
Luis Galárraga\inst{2},
Julien Delaunay\inst{2},
Laurence Rozé\inst{3},
Vaishnavi Bhargava\inst{4}
}
\authorrunning{Gaudel et al.}

\institute{ (corresponding author) Univ. Rennes, Ensai, CNRS, CREST, Rennes, France 
\email{romaric.gaudel@ensai.fr}\\
\and
Univ. Rennes, Inria, Irisa, France \email{\{julien.delaunay,luis.galarraga\}@inria.fr}\\
\and
Univ. Rennes, Insa, Inria, Irisa, Rennes, France
\email{laurence.roze@insa-rennes.fr} \\
\and
(during research) Inria/Irisa, Rennes, France~\email{vaishnavi.bhargava2605@gmail.com}
}

%
%\authorrunning{Gaudel et al.}
%ENSAI, France
%\email{romaric.gaudel@ensai.fr} \\
%\and 
%\institute{Inria/Irisa, France
%\email{\{vaishnavi.bhargava,julien.delaunay,luis.galarraga\}@inria.fr}\\
%\and
%Insa/Irisa, France
%\email{romaric.gaudel@ensai.fr} \\
%\and 
%\institute{??}
%\email{vaishnavi.bhargava@gmail.com}\\
%}

\maketitle
% IDA'22: Authors of accepted papers are allowed to use one additional page (so papers up to 13 pages) 

\begin{abstract}
The benefit of locality is one of the major premises of LIME, one of the most prominent methods to explain black-box machine learning models. This emphasis relies on the postulate that the more locally we look at the vicinity of an instance, the simpler the black-box model becomes, and the more accurately we can mimic it with a linear surrogate. As logical as this seems, our findings suggest that, with the current design of LIME, the surrogate model may degenerate when the explanation is too local, namely, when the bandwidth parameter $\sigma$ tends to zero. %This phenomenon arises also when the vicinity used to learn the surrogate model is poorly defined.
Based on this observation, the contribution of this paper is twofold. Firstly, we study the impact of both the bandwidth and the training vicinity on the fidelity and semantics of LIME explanations. Secondly, and based on our findings, we propose \slime, an extension of LIME that reconciles fidelity and locality. 
%Finally, we propose a stable solution to compute the importance weights associated to the features in an explanation.

%Based on these observations, the contribution of this paper is four-fold. Firstly, we revisit the semantics of the explanations delivered by LIME with respect to the properties of the black-box model. Secondly, we study the impact of both the kernel width and the training vicinity on the fidelity of explanations. Thirdly, we propose a framework to help users calibrate LIME more wisely, as well as to assess its viability as explanation method for a given scenario. Finally, we propose a stable solution to compute the importance weights associated to the features in an explanation.

%One of the ingredients of the framework \emph{Local Interpretable Model-Agnostic Explanations} (LIME) is the \emph{locality} which enables the derivation of a surrogate model both interpretable and faithful with the black-box model to interpret. Intuitively, the more we look locally, the simplest the black-box model. However, with the current design of LIME, the surrogate model degenerates when the band-with parameter controlling the locality tends to zero. In current paper we demonstrate how the band-with parameter more generally greatly impacts the quality of the output of LIME. We also propose a framework to help the user to set the right value for the band-with. Finally, we propose an alternative solution to compute the weight of interpretable features which proves to be more stable.
\end{abstract}


\section{Introduction}
The pervasiveness of complex automatic decision-making nowadays 
%in almost every domain of human activity 
has raised multiple concerns about the implications of AI for the values of fairness, trust, transparency, and privacy
\cite{bodria2021benchmarking,AI-under-law,remote-explainability}.
%\cite{bodria2021benchmarking,quantifying-interpretability,right-level-explanation,counterfactual-without-opening,remote-explainability,AI-under-law}.
These concerns have propelled a plethora of work in explainable AI, a domain concerned with the design of models that can provide high-level comprehensive explanations for their answers. These models can be either explainable-by-design, or rely on external modules that compute explanations \emph{a posteriori}. This need for post-hoc explainability is particularly compelling for sophisticated machine learning models, e.g., neural networks, whose logic is perceived as a black box by lay users.  

One of the most prominent modules to compute post-hoc explanations for black-box supervised ML models is LIME~\cite{lime}. This approach builds upon the notion of \emph{local feature attribution} via a \emph{linear surrogate}. Feature attribution means that the explanation quantifies the contribution of a set of features to the black box's answer. This allows users to build a ranking of the features that play the biggest role in the model's logic. 
We say the explanation is local because it only holds for a \emph{target instance} and its vicinity. By focusing on a region of the feature space, LIME reduces the complexity of the black box and can approximate it using a surrogate sparse linear function whose coefficients constitute the feature attribution scores of the explanation.
To learn this surrogate, LIME constructs a training set by generating artificial instances -- called neighbors -- around the target instance, and labeling them using the black box. The neighbors may not lie in the original feature space, but rather on a \emph{surrogate space} that is meaningful to humans, e.g., image segments instead of pixels for images.
The neighbors are weighted using an exponential kernel that depends on the distance to the target and a \emph{bandwidth} parameter $\sigma \in \mathbb{R}^+$. The weighting process controls the level of locality of the explanation:
the smaller $\sigma$ is, the more local the explanation becomes as closer neighbors are weighted higher than farther ones. More locality also implies focusing on a smaller region where the black box is presumably easier to approximate. 

As logical as this sounds, our experiments suggest that small values of $\sigma$ can yield unfaithful or even trivially empty explanations. This counter-intuitive result has thus motivated this work, which brings two contributions: (a) A study of the impact of the bandwidth and the training vicinity on the fidelity and semantics of LIME, namely the meaning of the feature attribution scores\footnote{By \emph{semantics of LIME}, we mean the information carried by the feature attribution scores.}; and (b) \slime{}, an extension of LIME that can solve the locality-fidelity paradox.
%where we study the influence of both the degree of locality and the surrogate space on the fidelity and semantics of LIME explanations. In addition, we propose a solution to this paradox in order to reconcile locality and fidelity in LIME. All in all, our contributions are:
%\begin{enumerate}
    %\item A discussion of the semantics of the explanations delivered by LIME in contrast to other linear explanation methods such as SHAP~\cite{shap}, and DeepLift~\cite{deeplift}
%    \item A study of the impact of the bandwidth and the training vicinity on the fidelity and semantics of LIME; 
%   \item \slime{}, an extension of LIME that can solve the locality-fidelity paradox. 
%    \item A stable solution to compute the coefficients in LIME explanations
%\end{enumerate}

This paper is structured as follows. In Section~\ref{sec:preliminaries} we introduce some background concepts and notations. We elaborate on our contributions in Sections~\ref{sec:effects-locality} and~\ref{sec:smoothed-lime}.
Section~\ref{sec:experiments} presents an experimental evaluation of \slime. In Section~\ref{sec:related-work-julien} we survey the state of the art. Section~\ref{sec:conclusion} concludes the paper.

\section{Preliminaries}
\label{sec:preliminaries}
%\subsection{Black Boxes and Local Linear Explanations}
%This paper is about explaining black-box models and functions using local surrogate functions as explanations. 
\paragraph{Black Boxes and Linear Surrogates. } We assume our black box is a classification function $f : \mathbb{R}^d \rightarrow \mathbb{R}$ ($d \in \mathbb{Z}^+$) that predicts the probability that a target instance $x \in \mathbb{R}^d$ belongs to a given class. We denote by $x[i]$ the $i$-th feature of $x$. Conversely, the explanation $g: \mathbb{R}^{\hat{d}} \rightarrow \mathbb{R}$ ($\hat{d} \in \mathbb{Z}^+$) is a linear surrogate function that approximates $f$ in the locality of $x$, i.e., $g(\hat{x}) = \hat{\alpha}_0 + \sum_{1 \le i \le \hat{d}}{\hat{\alpha}_i \hat{x}[i]}$. Note that $g$ may be defined on a \emph{surrogate space} different from $f$'s. %For example, explanations for image classifiers may be defined on the contributions of image segments (called super-pixels in~\cite{lime}) to the classifier's answer, whereas the original classifier may actually represent images as arrays of pixels or feature maps. 
This implies the existence of a conversion function $\eta_x: \mathbb{R}^{\hat{d}} \rightarrow \mathbb{R}^{d}$ from the surrogate to the original space. 
%Likewise, ML models on text can be explained on a surrogate space defined by the presence/absence of individual words. 

%We also highlight that the conversion function $\eta$ is not necessarily bijective, put differently, an instance $\hat{x} \in \mathbb{R}^{\hat{d}}$ in the surrogate space could be associated to an infinite number of instances $x$ in the original space $\mathbb{R}^{d}$. 

%\subsection{Linear Explanations}
%\label{subsec:linear-explanations}
%An explanation based on linear attribution provides the \emph{individual} contributions of a set of interpretable features to the answer of a black box $f$ on a target instance $x$. This is achieved by learning a linear surrogate model of the form $g(\hat{x}) = \hat{\alpha}_0 + \sum_{1 \le i \le \hat{d}}{\hat{\alpha}_i \hat{x}[i]}$ so that $\hat{\alpha}_i$ ($i\ge1$) denotes the contribution of the $i$-th feature to the value of $f(x)$. The meaning of $\hat{\alpha}_i$ depends both on the method to learn the surrogate and the properties of the black box. We elaborate on those semantics for the most relevant approaches for linear explanations in the next section.

%\subsection{LIME}
%\label{subsec:lime}
\paragraph{LIME.}
In \cite{lime}, the authors propose a model-agnostic method to compute local explanations
for ML models in the form of sparse linear surrogates. LIME learns an explanation $g$ for a black box $f$ and an instance $x$ by solving the following minimization problem:
%\begin{equation} g = \mathop{argmin}_{g \in \mathcal{G}}{\;\mathcal{L}_x(f, g)} + \Omega(g) \;\; \text{s.t. } \lVert \hat{\alpha_{g}} \rVert_0 < k \label{eq:lime} \end{equation}
\begin{equation} g = \mathop{argmin}_{g \in \mathcal{G}:\; \lVert \hat{\bm{\alpha}} \rVert_0 \leqslant k}{\;\mathcal{L}_x(f, g)} \label{eq:lime} \end{equation}
%\commentr{alpha en gras}
\noindent In other words, the surrogate $g$ is chosen such that it minimizes the error $\mathcal{L}_x$ w.r.t. the answers of $f$ on a neighborhood $\mathcal{X}$ around a target instance $x$.
To keep the explanation meaningful to humans, LIME restricts itself to surrogate functions $g$ with less than $k$ non-zero parameters, where $k$ is a user-configurable hyper-parameter set by default to $6$.
LIME does not assume access to the training data of the black box\footnote{The exception to this rule is its implementation for tabular data.}, therefore the neighbors $z\in\mathcal{X}$ take the form $z=\eta_x(\hat{z})$ where $\hat{z}\in \hat{\mathcal{X}}\subseteq\{0, 1\}^{\hat{d}}$ is a synthetic instance that lies on a binary space. This space is interpretated as the presence or absence of features of the target $x$, so that $x=\eta_x(\hat{x})$ with $\hat{x} = \mathds{1}^{\hat{d}}$. The neighbors in $\hat{\mathcal{X}}$ are obtained by toggling off bits in $x$'s binary representation $\hat{x}$. When a bit is set to zero in the surrogate space, the conversion function $\eta_x$ must map the resulting vector to the original space. For images, this can be achieved by replacing the toggled-off super-pixels with a baseline monochrome segment or with a patch from another image \cite{anchors}.
%For instance, if the black-box $f$ is a classifier for RGB images, those surrogate features could consist of predefined super-pixels.
%Note that the target instance $x$ is represented by the vector $\hat{x} = \mathds{1}_{\hat{d}}$.
%
%The neighbors in $\mathcal{X}$ do not have the same importance when learning the surrogate. 
%Instead, 
LIME weighs the neighbors in $\hat{\mathcal{X}}$ according to a kernel function $\pi^{\sigma}_x$ (based on a distance $D$ and a \emph{bandwidth} hyper-parameter $\sigma \in \mathbb{R}^+$) on the surrogate space, that is,
%Finally, the term $\Omega(g)$ in Equation~\eqref{eq:lime} penalizes complex surrogate models. 
%In summary: 
%LIME sets $\mathcal{L}_x$ to a weighted squared loss, the surrogate neighbors $\hat{z}$ to random samples in $\{0, 1\}^{\hat{d}}$, $\pi_x$ to an exponentially decreasing kernel on a distance $D$, and $\Omega$ to a function that penalizes linear models with more than $k$ non-zero coefficients, that is: \\
%\[ \mathcal{L}_x(f, g) = \sum_{\hat{z}\in\hat{\mathcal{X}}}{\pi_x(\hat{z})(f(\eta_{x}(\hat{z})) - g(\hat{z}))^2} \]\vspace{-1em} 
\begin{align*}
    %\pi^{\sigma}_x(\hat{z}) &= \exp\left(\nicefrac{-D(\hat{x}, \hat{z})^2}{\sigma^2}\right)
    \mathcal{L}_x(f, g) &= \sum_{\hat{z}\in\hat{\mathcal{X}}}{\pi_x^\sigma(\hat{z})(f(\eta_{x}(\hat{z})) - g(\hat{z}))^2},
    &
    \text{with } \pi^{\sigma}_x(\hat{z}) &= \exp\left(\nicefrac{-D(\hat{x}, \hat{z})^2}{\sigma^2}\right).
    %\Omega(g) &= \infty \mathbbm{1}\lbrack\;|F_g | > k \; \rbrack
    %\Omega(g) &= \lVert \hat{\alpha}_g \rVert_0 - k 
\end{align*} 
The hyper-parameter $\sigma$  controls the locality of the explanation so that smaller values give more weight to the instances that lie close to $\hat{x}$, i.e., those instances with fewer toggled-off bits. 
%\noindent The neighbors in $\hat{\mathcal{X}}$ are obtained by toggling off bits in the target's binary representation $\hat{x}$. When a bit is set to zero in the surrogate space, the conversion function $\eta_x$ must map the resulting vector to the original space. For images, this can be achieved by replacing the toggled-off super-pixels with a baseline grayed out segment\footnote{Some other alternatives are proposed in \cite{anchors}.}.
LIME does not make any assumptions about the inner-workings of $f$, however the distance $D$ and the conversion functions $\eta_x$ depend on $f$'s original space, which at the same time depends on the instances' data type. 
% Moreover,~\cite{lime} does not discuss about the influence of other arguments such that the bandwidth $\sigma$, the distance $D$, or the neighborhood \neighbors. 

\paragraph{Quality Metrics. }
The quality of the local surrogate $g$ is evaluated in terms of its \emph{fidelity}, which can be measured via the surrogate's adherence to the black box $f$ in the vicinity of $x$. Adherence is usually measured via the coefficient of determination R$^2$~\cite{iLIME,ALIME,OptiLIME}. The R$^2$ score %In a setting where the classifier outputs probabilities,
measures the similarity between the predictions of both functions, compared to the variance of the black-box prediction. This coefficient lies in $(-\infty, 1]$, where $R^2=1$ means $g$ fits $f$ perfectly and $R^2=0$ (respectively $R^2<0$) implies that $g$ is as accurate as (resp. less accurate than) the best constant model.
%This coefficient is smaller than 1, with $R^2=1$ meaning that the $g$'s prediction is equal to the black-box one, and $R^2=0$ (respectively $R^2<0$) meaning that $g$ makes the same error w.r.t. $f$ as (resp. worse error than) the best constant model.
%Hence, the greater $R^2$, the better.
%For pure classification tasks, the fidelity can be the accuracy or the AUC-ROC of $g$ w.r.t. the output of $f$. For regression problems or classification tasks that predict the likelihood of belonging to a class, 
%which can be measured via the error w.r.t. the black box's predictions in the locality where the surrogate was learned. We argue that unbounded error metrics -- such as the RMSE -- are not suitable to quantify fidelity unless they are provided with a baseline. To that end, we propose the R$^2$ score as quality metric. 
%That way, values close to 1 denote a good linear fit. 

\noindent When a gold standard set $F_f(x)$ of important features is available, we can also calculate fidelity as the agreement between the explanation and the gold standard. This can be quantified via metrics such as \emph{recall}~\cite{lime}, \emph{precision}, or \emph{coverage}~\cite{leftist}. 
Assuming the surrogate and the original feature spaces are identical, if the explanation $g$ for the target instance $x$ reports features $F_g(x)$ as the most important, the recall and precision of $g$ are respectively $\frac{|F_f(x) \cap F_g(x)|}{|F_f(x)|}$  and $\frac{|F_f(x)\cap F_g(x)|}{|F_g(x)|}$.
Coverage can be used for data types where segments, i.e., conglomerates of contiguous features, are more meaningful to humans than individual features. Examples are time series and images. For those cases, the coverage is the proportion of the gold standard regions that overlap with the regions reported by the surrogate. 
Further specialized metrics have been proposed to measure the fidelity of pixel attribution explanations for image classifiers~\cite{jia2021studying}.

%\section{Related Work}
%\label{sec:related-work}
%\input{related-work}

\section{Locality vs. Fidelity}
\label{sec:effects-locality}

\begin{figure*}[t!]
	\centering
%	\begin{subfigure}[b]{0.15\textwidth}
%		\includegraphics[width=\textwidth]{prediction_probabilities.png}
%		\caption{{\footnotesize Prediction}}
%		\label{fig:prediction_probabilities}
%	\end{subfigure}
	\begin{subfigure}[b]{0.32\textwidth}
		\includegraphics[width=\textwidth]{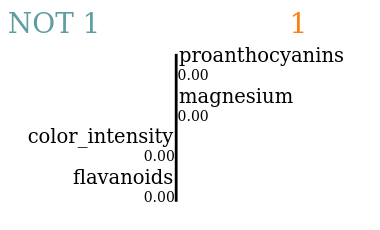}
		\caption{$\sigma=0.1$}
		\label{fig:first_explanation}
	\end{subfigure}
	\begin{subfigure}[b]{0.32\textwidth}
		\includegraphics[width=\textwidth]{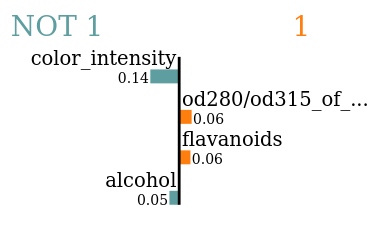}
		\caption{$\sigma=0.75$}
		\label{fig:second_explanation}
	\end{subfigure}
	\begin{subfigure}[b]{0.32\textwidth}
		\includegraphics[width=\textwidth]{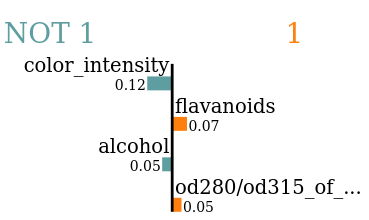}
		\caption{$\sigma=100$}
		\label{fig:third_explanation}
	\end{subfigure}
	\caption{LIME explanations for three different bandwidths on the same instance of the wine dataset ($k=4$).}
	\label{fig:explanations}
\end{figure*}

In this section we study the impact of two important elements of LIME on the fidelity and semantics of explanations, namely the bandwidth $\sigma$ and the neighborhood \neighbors.  
%By default, LIME outputs the 6 features with the highest absolute weight, however the explanation quality is calculated on the entire set of features in the surrogate model.

\subsection{The Paradox of Small Bandwidth}
\label{sec:impact_kernel}
We illustrate the impact of $\sigma$ on the output of the tabular variant of LIME\footnote{The discretization is off, hence the classifier and the explanation operate in the same space.}, which we use to explain a random forest classifier trained on the UCI wine dataset\footnote{\url{https://archive.ics.uci.edu/ml/datasets/wine}}. %The classifier distinguishes among three wine types based on the results of 13 chemical analyses. 
%For illustration purposes, we resort to a simplified implementation of tabular LIME on synthetic 2D data\footnote{The datasets are generated using the functions \texttt{make\_blobs} and \texttt{make\_moons} of scikit-learn} for binary classification.
Tabular LIME sets $\sigma=0.75$ with no further explanation. Changing $\sigma$ can, however, drastically change the resulting explanation as depicted in Figure~\ref{fig:explanations}. In particular, LIME computes null attribution coefficients when $\sigma=0.1$. Changing $\sigma$ from 0.75 to 100 rearranges the attribution ranking of the features.  
%The charts illustrate three LIME explanations on the same instance when $\sigma=\{0.1, 0.75, 4\}$.

To investigate the cause of this instability, we measure the adherence of the surrogate in \neighbors~as $\sigma$ varies 
%in $\{10^{-4}, \dots, 10^{4} \}$ 
for all the test instances of the dataset. 
%We plot the R$^2$ score of the LIME surrogate in a test weighted neighborhood \neighbors~for two of those instances in Figure~\ref{fig:first_example}, where instance 2 corresponds to the instance explained in Figure~\ref{fig:explanations}. The score is calculated on a test set of artificial neighbors generated by LIME.    
We plot the results for two instances in Figure~\ref{fig:first_example}, where instance 2 is the example explained in Figure~\ref{fig:explanations}.
%\begin{figure*}[t!]
%	\centering
%	\begin{subfigure}[b]{0.48\textwidth}
%	\footnotesize
%		\includegraphics[width=0.9\textwidth,height=3cm]{instances_1_and_2.png}
%		\caption{R$^2$ vs. $\sigma$}
%		\label{fig:first_example}
%	\end{subfigure} \smallskip  
%	\begin{subfigure}[b]{0.48\textwidth}
%	\subfloat[$\sigma=0.1$]{
%   \includegraphics[width=0.8\columnwidth]{weights_sigma_0_1.png}
%        	\label{fig:weights_sigma_01}
%    }
    
%   \subfloat[$\sigma=100$]{
%    \includegraphics[width=0.8\columnwidth]{weights_sigma_100.png}
%    \label{fig:weights_sigma_100}
%    }
	%\caption{Distribution of the weights for the LIME neighborhood of instance 2%}
%	\end{subfigure}
%	\caption{Left: Impact of the bandwidth $\sigma$ on the R$^2$ score of LIME for two instances of the wine dataset. Right: Distribution of the neighborhood weights for instance 2.}
%	\label{fig:examples}
%\end{figure*}

\begin{figure*}[t!]
	\centering
	\begin{subfigure}[b]{0.40\textwidth}
	\footnotesize
		\includegraphics[width=0.9\textwidth]{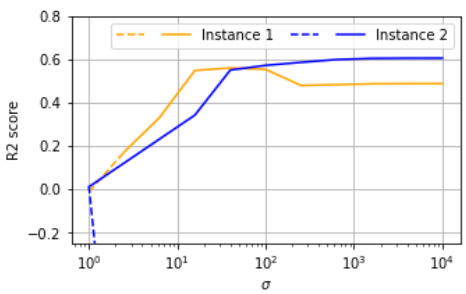}
		\caption{R$^2$ vs. $\sigma$}
		\label{fig:first_example}
	\end{subfigure} % \smallskip  
	\begin{subfigure}[b]{0.28\textwidth}
    \includegraphics[width=0.88\columnwidth]{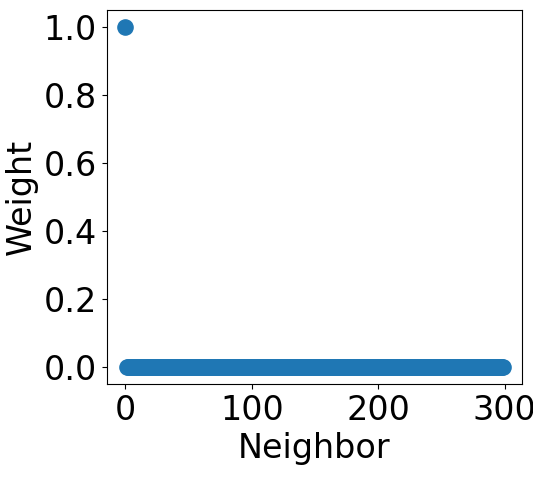}
        \caption{$\sigma=0.1$}

        	\label{fig:weights_sigma_01}
    \end{subfigure}
	\begin{subfigure}[b]{0.28\textwidth}    
    \includegraphics[width=0.88\columnwidth]{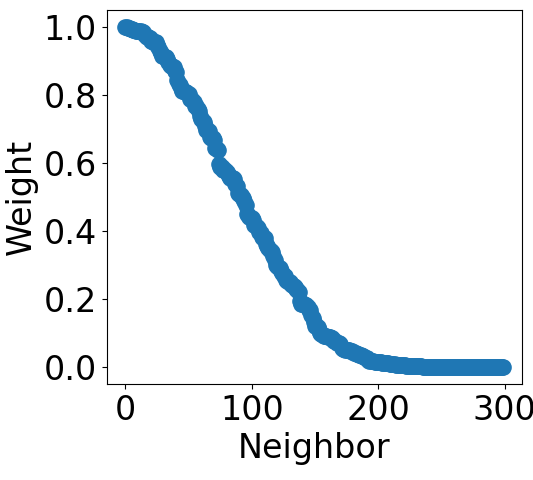}
    \caption{$\sigma=100$}
    \label{fig:weights_sigma_100}
    
	\end{subfigure}
	\caption{Left: Impact of the bandwidth $\sigma$ on the R$^2$ score of LIME for two instances of the wine dataset. Right: Distribution of the neighborhood weights for instance 2.}
	\label{fig:examples}
\end{figure*}

We recall that the R$^2$ score is calculated as $1 - \nicefrac{v_r(g)}{v(f)}$, where $v_r(g)$ is the residual sum of squares of the surrogate $g$ and $v_r(f)$ is the total sum of squares of $f$'s answers. 
This means that the surrogate accounts for no more than 60\%
%the R$^2$ is no higher than 0.6 for our example instances. Similar results are observed when we average across all test instances. This means that the surrogate models learned by LIME account for at most 60\%
of the variability of the black box in $\hat{\mathcal{X}}$. The dashed regions of the curves indicate that the surrogate model has degenerated into a set of zero weights. This points out a counter-intuitive phenomenon: higher locality -- achieved by making $\sigma$ small -- yields poor explanations. We also observe that the R$^2$ may not increase monotonically with $\sigma$. Based on these observations, we devise two research questions that drive our contribution: (i) Why do seem locality and fidelity in opposition?, and (ii) what makes a good LIME explanation? 
%\begin{enumerate}
%    \item Why do seem locality and fidelity in opposition?
%    \item What makes a good LIME explanation?
%    \item How local should a good explanation be?  
%\end{enumerate}

\subsection{Why do Seem Locality and Fidelity in Opposition?} 
We investigate the cause of this paradox by means of Figures~\ref{fig:weights_sigma_01} and \ref{fig:weights_sigma_100} that depict the distribution of weights for the neighbors of instance 2 for $\sigma=0.1$ and $\sigma=100$. In the first case, the LIME surrogate is a degenerated model that predicts a constant as hinted by Figure~\ref{fig:first_example} and its corresponding explanation in Figure~\ref{fig:first_explanation}. Figure~\ref{fig:weights_sigma_01} tells us that the bulk of the weights is concentrated on the target instance. Such a phenomenon leads to a trivial training set. Even though locality is defined in terms of the entire set of instances in $\hat{\mathcal{X}}$, almost all of them are dispensable because they do not have any influence when learning the surrogate. The situation is less skewed for $\sigma=100$ (Figure \ref{fig:weights_sigma_100}), which yields the non-trivial explanation in Figure~\ref{fig:third_explanation}. 

From this analysis we conclude that the selection of $\sigma$ and the construction of $\hat{\mathcal{X}}$ must go in hand. We thus propose a strategy to jointly select them in Section~\ref{sec:smoothed-lime}.

%\begin{figure*}[hbt]
%	\centering
%	\begin{subfigure}[b]{0.48\textwidth}
%		\includegraphics[width=\textwidth]{weights_sigma_0_1.png}
%		\caption{$\sigma=0.1$}
%		\label{fig:weights_sigma_1}
%	\end{subfigure}
%	\begin{subfigure}[b]{0.48\textwidth}
%		\includegraphics[width=\textwidth]{weights_sigma_10.png}
%		\caption{$\sigma=10$}
%		\label{fig:weights_sigma_2}
%	\end{subfigure}
%	\caption{Distribution of the weights for the neighborhood of instance 2 in Figure~\ref{fig:second_example}. \emph{x-axis}: instances in \neighbors~sorted by weight. \emph{y-axis}: weight assigned by the kernel $\pi_{\hat{\mathcal{X}}}$ }
%	\label{fig:weights}
%\end{figure*}

%\begin{figure*}[hbt]
%	\centering
%	\begin{subfigure}[b]{0.48\textwidth}
%		\includegraphics[width=\textwidth]{artificial_bbox_2_average}
%		\caption{}
%		\label{fig:artificial_bbox_2_average}
%	\end{subfigure}
%	\begin{subfigure}[b]{0.48\textwidth}
%		\includegraphics[width=\textwidth]{artificial_bbox_2_gaussian}
%		\caption{}
%		\label{fig:artificial_bbox_2_gaussian}
%	\end{subfigure}
%	\caption{Two LIME explanations for the black box and instance in Figure~\ref{fig:artificial_bbox_2}. On the left: degenerated explanation using a $\eta$ conversion function that maps zeros in the surrogate space to constant values in the original space. On the right: the conversion function $\eta$ maps zeros in the surrogate space to values sampled according to a Gaussian distribution. }
%	\label{fig:artificial_bbox_2_explanations}
%\end{figure*}

\subsection{What Makes a Good LIME Explanation?}
\begin{figure*}[t!]
    \captionsetup[subfigure]{aboveskip=-1pt,belowskip=-1pt}
	\centering
	\begin{subfigure}[b]{0.32\textwidth}
		\includegraphics[width=\textwidth]{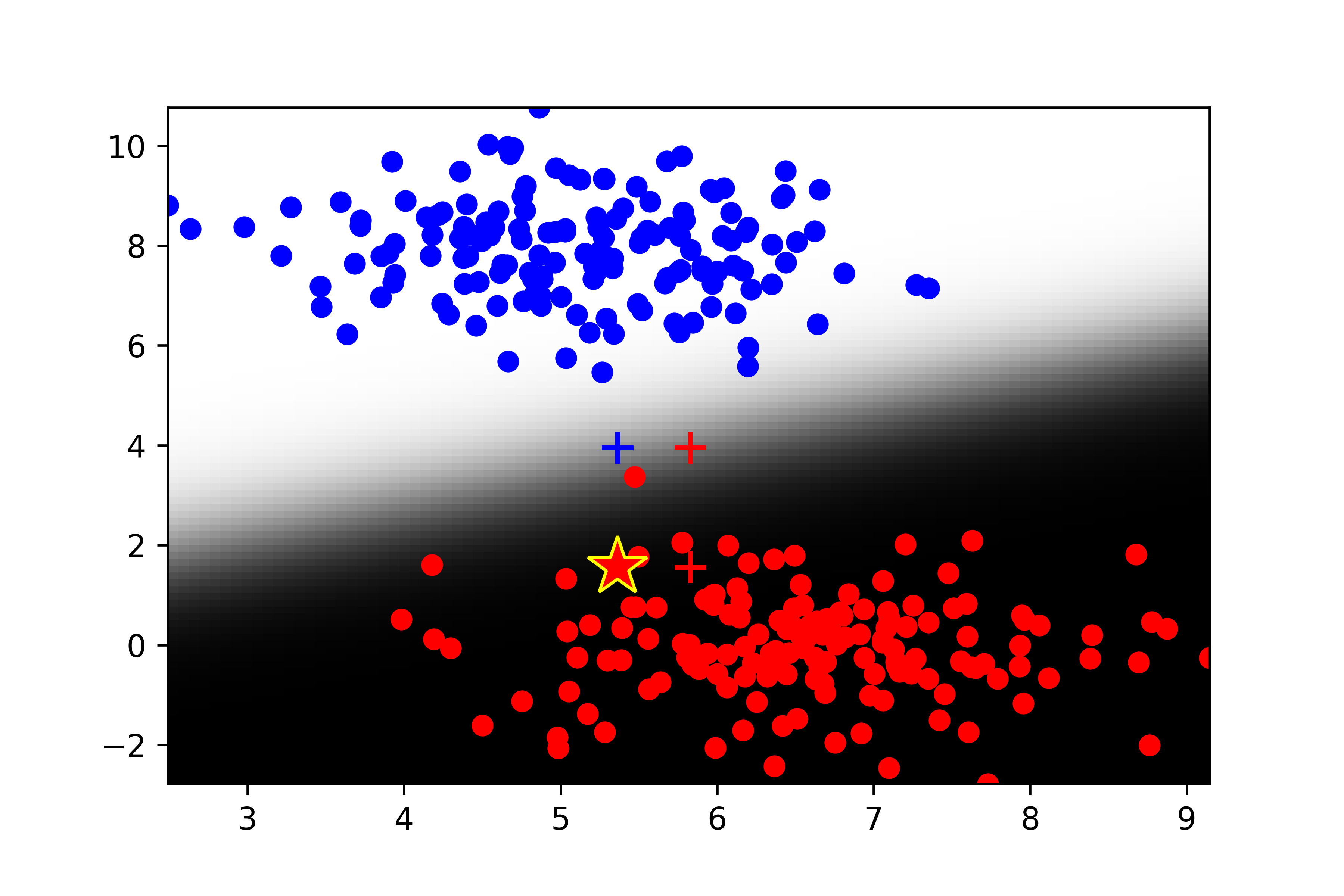}
		\caption{Logistic Regression}
		\label{fig:artificial_bbox_1}
	\end{subfigure} %\quad
	\begin{subfigure}[b]{0.32\textwidth}
	\includegraphics[width=\textwidth]{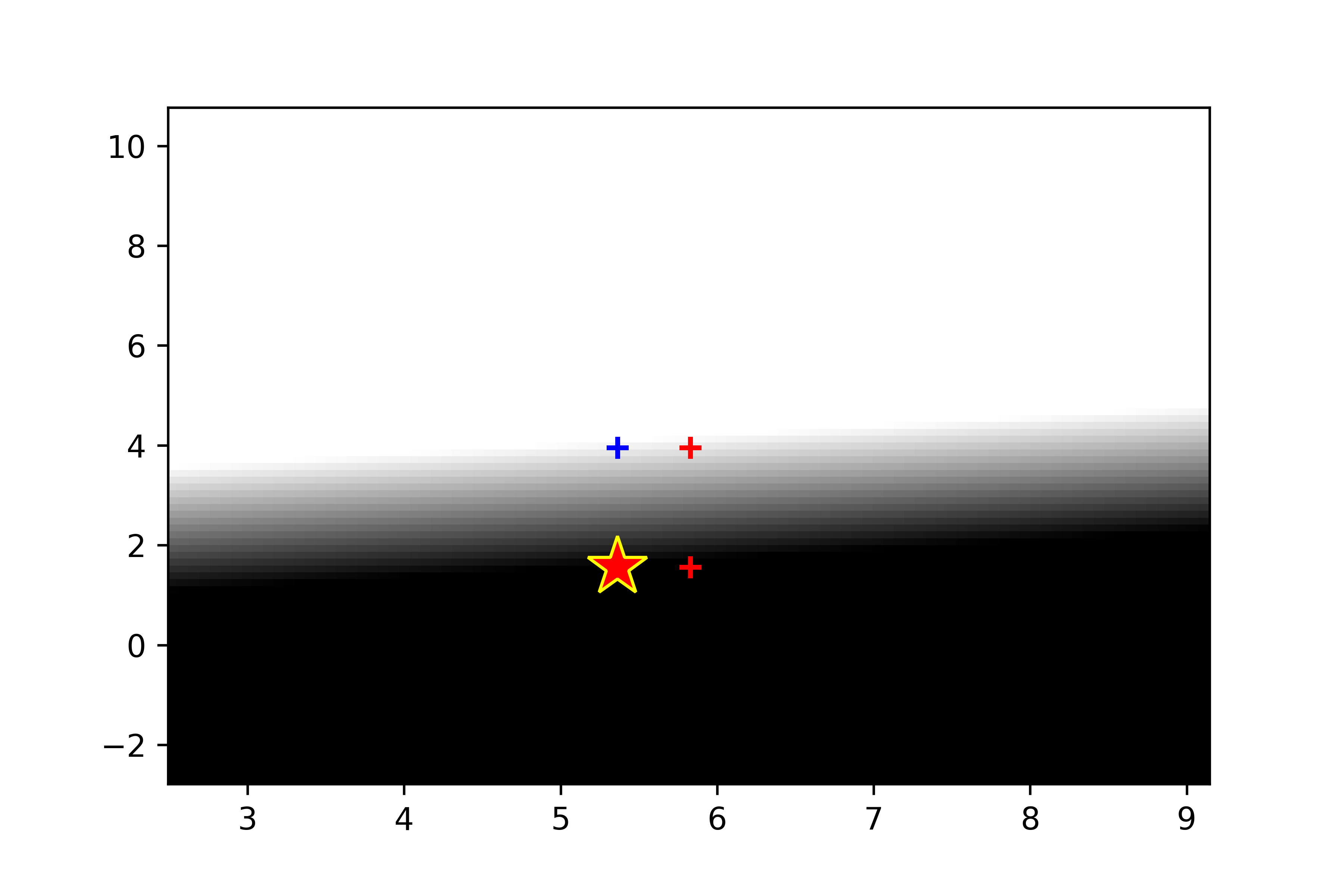}
	\caption{$\sigma=0.5$}
    \label{fig:artificial_bbox_1_sigma_0.5}
	\end{subfigure}	
	\begin{subfigure}[b]{0.32\textwidth}
	\includegraphics[width=\textwidth]{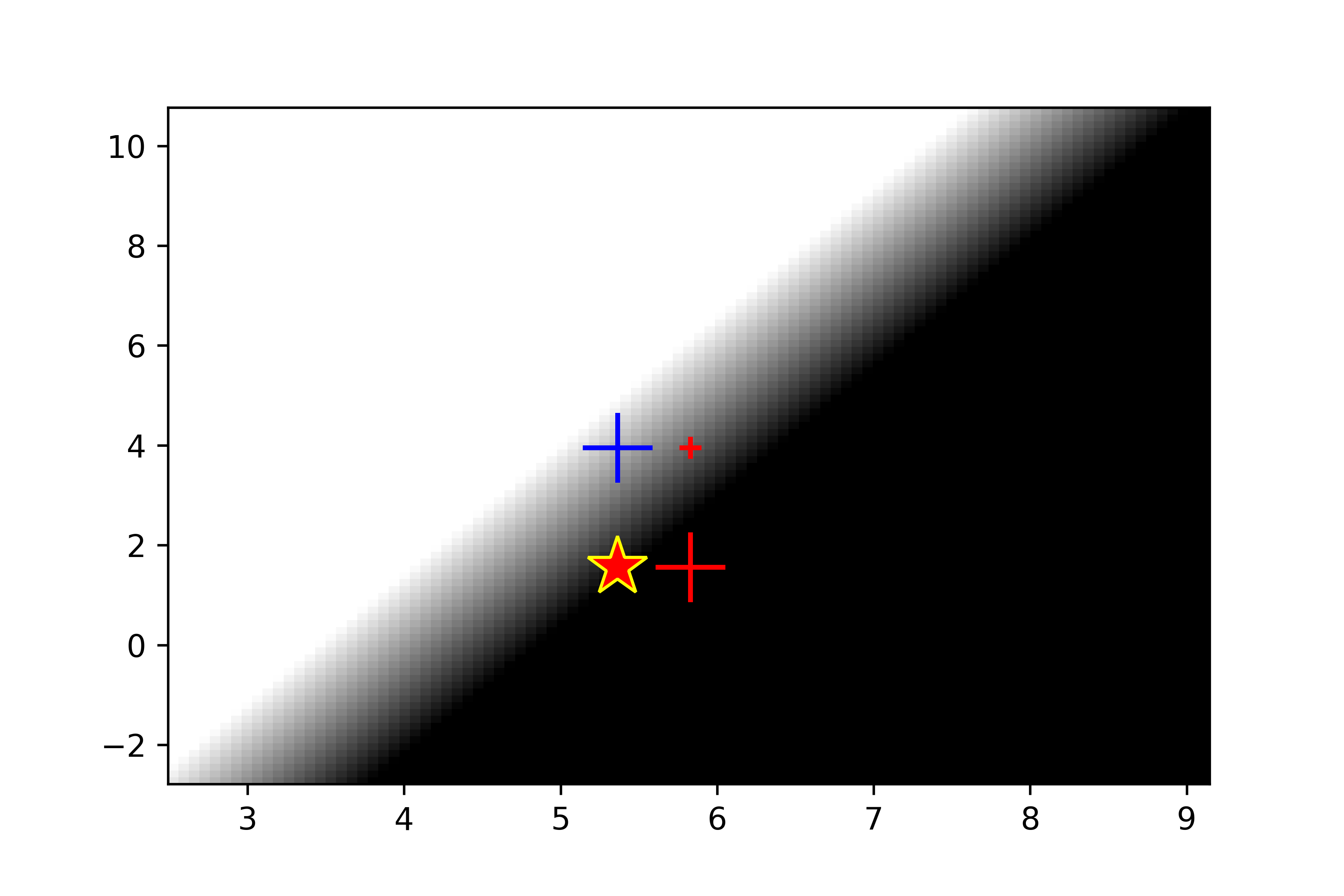}
		\caption{$\sigma=1.0$}
    \label{fig:artificial_bbox_1_sigma_1}
	\end{subfigure}		
%	    \subfloat[$\sigma=0.5$]{
%		    \includegraphics[clip,width=0.8\columnwidth,height=2.5cm]{artificial_bbox_1_sigma_1.png}
%		    \label{fig:artificial_bbox_1_sigma_1}
%		}\vspace{-0.08cm}

%	    \subfloat[$\sigma=1.0$]{
 %   		\includegraphics[clip,width=0.8\columnwidth,height=2.5cm]{artificial_bbox_1_sigma_2}
  %  	\label{fig:artificial_bbox_1_sigma_2}
%    	}%\vspace{-0.3cm}
%	\begin{subfigure}[b]{0.48\textwidth}
%		\includegraphics[width=\textwidth]{artificial_bbox_2}
%		\caption{Multilayer Perceptron}
%		\label{fig:artificial_bbox_2}
%	\end{subfigure}
	\caption{
	%Two black boxes trained on 2D artificial data with a target instance (marked as a star). 
	Left: A logistic regression classifier and a neighborhood (denoted by + marks) generated on a 2D discrete surrogate space. 
	%where $\eta(\hat{x}[i]) = \mu_i$ whenever $\hat{x}[i] = 0$, and $\mu_i$ is the empirical mean of $i$-th feature in the data.
	Center and right: Two LIME explanations.
	The gradient of each of these functions at the target example  (denoted by the * mark) is orthogonal to the border between white area and black area.
	The explanation in the middle captures the black box's gradient more faithfully.
	%On the right: A multilayer perceptron classifier and a target instance.
	}
		\label{fig:artificial_black_boxes}

\end{figure*}

The human aspects of interpretability are beyond the scope of this paper; instead this study is concerned with the quality and meaningfulness of explanations from a mathematical point of view. 
As suggested by \cite{explaining-lime}, LIME computes a scaled version of the gradient $\nabla{}f$ for linear black boxes $f$. The scaling arises because the surrogate is learned on a finite number of neighbors in a discrete space, and the scaling factor   %-- tabular LIME discretizes the continuous features by default. 
%Indeed, \cite{explaining-lime} shows that the factor 
depends on $x$, $\sigma$, $\eta_x$, and \neighbors. We argue that in the absence of a reference instance (as in~\cite{shap,deeplift,ig}), 
explanations based on instantaneous gradients are meaningful and desirable because their semantics are well-defined: the \emph{surrogate gradient} $\hat{\nabla{}}f(x)$ is the contribution of each surrogate feature to $f$'s change rate at point $x$. 
%In this line of thought, linear explanations that estimate $\hat{\nabla{}}f(x)$ are meaningful and therefore desirable. 
That said, LIME does not always estimate $\hat{\nabla{}}f$ accurately as suggested by Figure~\ref{fig:artificial_black_boxes}. The figures show that the weights associated to the neighbors may yield an estimation that differs largely from the black box's actual gradient in Figure~\ref{fig:artificial_bbox_1}.  

\section{\slime}
\label{sec:smoothed-lime}

%In Section \ref{sec:impact_kernel} we showed that LIME can produce degenerated explanations when the bandwidth $\sigma$ tends to zero, because in that case LIME looks only at a small part of the neighborhood around the target. 

%To tackle this issue, we introduce an extension, called
To tackle the locality-fidelity paradox explained in Section \ref{sec:impact_kernel}, we introduce an extension of LIME, called
\slime{} (\emph{Smoothed LIME}), that we detailed in the following.

\subsection{Generic Algorithm}

 \begin{algorithm}[t!]
     \caption{\slime{} \hfill applied to black-box function $f$ at target instance $x$}\label{alg:slime}
     \begin{algorithmic}[1]
         \Require{Conversion function $\eta_x$, distribution $\nu_\sigma$ on the surrogate space}
         \Require{Number $k$ of features in the explanation, number $n$ of local examples}
         \State $\hat{\mathcal{X}} \gets \left\{\hat{z}^{(i)} : i = 1,\dots, n\right\}$, where 
         $\hat{z}^{(i)} \sim \nu_\sigma$ \textbf{for} $i = 1,\dots, n$
         \State \Return  $\mathop{argmin}_{g \in \mathcal{G}: \; \lVert \hat{\bm{\alpha}} \rVert_0 \leqslant k}
         \sum_{\hat{z} \in \hat{\mathcal{X}}} \left(f(\eta_x(\hat{z})), g(\hat{z})\right)^2
         $
     \end{algorithmic}
 \end{algorithm}

LIME may compute degenerated explanations due to two main factors: (i) the discreteness of the surrogate space, and (ii) the fact that instance generation and weighting are decoupled. Indeed, LIME first generates a discrete neighborhood $\hat{\mathcal{X}}$ (independently of $\sigma$), and then weighs the instances in $\hat{\mathcal{X}}$ using $\pi^{\sigma}_x$. In the extreme cases when $\sigma$ tends to zero, the weighting is concentrated on $\hat{x}$.
%In the extreme cases, when $\sigma$ tends to zero, the weight associated to these examples also tends to zero except for the examples $\hat{z}$ equal to the target's binary representation $\hat{x}$; in other words the weighting concentrates on the target alone.

To prevent this skewed concentration of weights, we control the locality of the explanation in a single step (see Algorithm~\ref{alg:slime}).
%we propose to directly handle the localisation of the surrogate neighbors.
Hence, we define the neighbors in the continuous space $[0,1]^{\hat{d}}$ and populate  $\hat{\mathcal{X}}$ with examples $\hat{z}$ whose distance $D$ to $\hat{x}$ is of \emph{the same magnitude} as $\sigma$. Concretely, the neighborhood $\hat{\mathcal{X}} = \{ \hat{z}^{(1)}, \dots,\hat{z}^{(n)} \}$ consists of $n$ equally-weighted instances
drawn independently from a distribution $\nu_\sigma$.
%, and each instance is assigned the same weight.
Such a design decision enables $g$ to approximate $\hat{\nabla{}}f$ when $\sigma$ tends to zero, without hindering interpretability: $g$ still combines the contributions of the surrogate features linearly, and we can still confer an interpretable meaning to the neighbors as later explained in Section~\ref{sec:concreteSLIME}.
%We argue that the latter effect is of less relevance as long as $g$ remains interpretable with clear semantic.
%Anyhow, for LIME and \slime{}, the neighborhood examples are a proxy to get the surrogate function $g$, the mitigation of their interpretability does not matters as soon as $g$ remains iterpretable.
%Concretely, the neighborhood $\hat{\mathcal{X}} = \{ \hat{z}^{(1)}, \dots,\hat{z}^{(n)} \}$ consists of $n$ instances
%drawn independently from a distribution $\nu_\sigma$, and each instance is assigned a weight of one.
Moreover, this allows controlling locality via the bandwidth of the neighborhood distribution, and not anymore through an a-posteriori weighting.
%That is, \slime{} controls the locality directly through the bandwidth of the neighborhood distribution, and not anymore through an a posteriori weighting process.

Note that \slime{} also requires the definition of new conversion functions $\eta_x$ as $\hat{\mathcal{X}}$ is now a subset of the continuous space $[0,1]^{\hat{d}}$ instead of the discrete space $\{0,1\}^{\hat{d}}$.
In Section~\ref{sec:concreteSLIME} we provide examples of proper distributions $\nu_\sigma$ and functions $\eta_x$ for images, time series, and tabular data.

\subsection{\slime{} Subsumes LIME}
%As expressed by the following Lemma, \slime{} can target the same loss as LIME.  % the sa w.r.t. LIME, as can be seen in following Lemma.
%Note that this coupling does not reduce the range of application of \slime{} w.r.t. LIME, as can be seen in following Lemma.
%Lemma \ref{th:coupling} states that there is a distribution $\nu_\sigma$ such that both expectations are minimized by the same function $g$. 

\begin{lemma}\label{th:coupling}
Let $f$ be a function and $x$ a target instance. There is a distribution $\nu_\sigma$ over $[0,1]^{\hat{d}}$  such that LIME and \slime{} are minimizing the same expected loss function.
\end{lemma}

\begin{proof}
LIME outputs a function $g$ that minimizes the loss $\mathcal{L}_x(f, g)$ which is the residual sum of squares of the examples drawn from a distribution $\nu$. The expectation of this loss function w.r.t. to a random neighborhood is
$\EE_{\hat{z}\sim\nu}\left[\pi_x^\sigma(\hat{z})\left(f(\eta_x(\hat{z}))- g(\hat{z})\right)^2\right].$
Remark that $\nu$ is a distribution on the finite space $\{0,1\}^{\hat{d}}$, then $\nu=\sum_{\hat{z} \in \{0,1\}^{\hat{d}}} w_\nu(\hat{z})\delta(\hat{z})$, where $\delta(\hat{z})$ is the Dirac distribution at point $\hat{z}$, and $w_\nu(\hat{z})$ is a positive real number.

\noindent Similarly, \slime{} returns the linear surrogate $g$ that minimizes a loss with expectation 
$\EE_{\hat{z}\sim\nu_\sigma}\left[\left(f(\eta_x(\hat{z}))- g(\hat{z})\right)^2\right].$
Let $Z$ be $\sum_{\hat{z} \in \{0,1\}^{\hat{d}}} \pi_x^\sigma(\hat{z})w_\nu(\hat{z})$. 
If we consider \slime{} with generating distribution $\nu_\sigma = \nicefrac{1}{Z}\sum_{\hat{z} \in \{0,1\}^{\hat{d}}} \pi_x^\sigma(\hat{z})w_\nu(\hat{z})\delta(\hat{z})$, then
\begin{align*}
\EE_{\hat{z}\sim\nu_\sigma}\left[\left(f(\eta_x(\hat{z}))- g(\hat{z})\right)^2\right]
&=\sum_{\hat{z} \in \{0,1\}^{\hat{d}}} \frac{\pi_x^\sigma(\hat{z}) w_\nu(\hat{z})}{Z} \left(f(\eta_x(\hat{z}))- g(\hat{z})\right)^2\\
&=\frac{1}{Z}\EE_{\hat{z}\sim\nu}\left[\pi_x^\sigma(\hat{z})\left(f(\eta_x(\hat{z}))- g(\hat{z})\right)^2\right],
\end{align*}
which concludes the proof.
\end{proof}

\begin{remark}
It follows from Lemma \ref{th:coupling} that \slime{} may be used as a placeholder for LIME. Still, the proposed distribution $\nu_\sigma$ is practical only when $d$ is small, or when $\nu_\sigma$ corresponds to a well-known distribution. Otherwise, storing the $2^{\hat{d}}$ coefficients $\pi_x^\sigma(\hat{z})w_\nu(\hat{z})$ is unpractical. Anyway, we demonstrate in Section~\ref{sec:experiments} that \slime{} with a continuous distribution is more faithful than LIME.
\end{remark}

\subsection{\slime{} and the Gradient of the Black-Box Function}
Let us assume the surrogate function $f\circ\eta_x$ to be differentiable at $\hat{x}$.
Let us also denote by $\hat{\bm{\alpha}}$ the weights of the linear model returned by \slime{} when we drop the sparseness constraint. Then for any family of continuous distributions $\nu_\sigma$ on $[0,1]^{\hat{d}}$, such that their mass concentrates on $\hat{x}$ when $\sigma$ tends to zero, $\hat{\bm{\alpha}}$ tends to the gradient $\hat{\nabla} f(x)$ of $f\circ\eta_x$ at point $\hat{x}$. An example of such family of distributions is the set $\{\mathcal{N}\left(\hat{x}, \sigma^2\pmb{I}\right), \sigma \in \mathbb{R}^+\}$ of Gaussian distributions centered at $\hat{x}$ with variance $\sigma^2\pmb{I}$, where $\pmb{I}$ is the identity matrix.

This property has two main implications. First, while LIME degenerates as $\sigma$ approaches zero, \slime{} remains well-defined for any value of $\sigma$. Secondly, we know what \slime{} is targeting when we look locally at $\hat{x}$:
$\hat{\nabla} f(x)$.
%the gradient of the surrogate function $f\circ\eta_x$ at point $\hat{x}$.
%Phrased differently, the linear approximation derived by \slime{} is the first order approximation of $f\circ\eta_x$ at point $\hat{x}$.

\begin{remark}
%There are at least two settings for which surrogate gradients are meaningless: piece-wise constant functions and functions which vary a lot\commentl{Shall we refer to them in a more strict way?}. A piece-wise constant function may be the result of black-box model implemented as a random forest \cite{breiman2001randomforest}. In such a case, \slime{} will output a zero gradient as soon as the bandwidth of the generating distribution is small enough. While the weights returned by \slime{} may be mathematically consistent for such kinds of black boxes, they are misleading for the user as they carry on information which is too local. Anyway, \slime{} also returns weights on interpretable features for higher values of $\sigma$, and these weights may prove to be more informative. \commentl{I do not understand the last sentence.}
There are settings for which surrogate gradients are meaningless: piece-wise constant functions such as random forests. %\cite{breiman2001randomforest}
%is an example of algorithm learning a piece-wise constant function. 
In such a scenario, \slime{} outputs a zero gradient as soon as the bandwidth of the generating distribution is small enough. While the weights returned by \slime{} are mathematically consistent for such kinds of models, they are useless as they carry on information that is too local. If that is the case, users may pick a higher value for $\sigma$, or resort to a rule-based surrogate~\cite{anchors}.
\end{remark}

\subsection{\slime{} Implementations}\label{sec:concreteSLIME}
Let us now discuss examples of concrete distributions $\nu_\sigma$ and functions $\eta_x$. % for image data, timeseries data and tabular data.
The generating distribution $\nu_\sigma$ is the same for image and time series datasets: the uniform distribution on $[1-\sigma,1]^{\hat{d}}$, with $\sigma \in (0,1]$.
As needed, this distribution concentrates around the surrogate target $\hat{x} = \mathds{1}^{\hat{d}}$ when $\sigma$ tends to zero.
%We note that $\nu_\sigma$ is not centered on $\hat{x}$, however this does not prevent \slime{} from returning a good linear approximation.

In regards to the conversion function $\eta_x$, we recall that for both images~\cite{lime} and time series~\cite{leftist}, LIME splits the original instance into $\hat{d}$ contiguous regions, namely super-pixels for images or fragments of fixed size for time series. Those regions define the features of the surrogate space. Given a neighbor $\hat{z} \in \hat{\mathcal{X}}$ and a surrogate feature $j$, we can project $\hat{z}$ back to the original space by interpolating the original features of the target $x$ with a baseline $x_0$, i.e., $\eta_x(\hat{z})[i] = (1-\hat{z}[j])x_0 + \hat{z}[j]x[i]$ for all the original features $i$, i.e., pixels or time measures, covered by segment $j$. We set $x_0=0$ in our experiments, i.e., the interpolation is done w.r.t. a black image and a null time series. 
%The function $\eta_x$ depends on the targeted data type.
%For images, \slime{} uses a smoothed version of LIME's conversion function to handle surrogate features in $[0,1]^{\hat{d}}$. The target image $x$ is split into a set of $\hat{d}$ meta-pixels and one surrogate feature is defined per meta-pixel. Let $i$ be a pixel in the original image and $j$ the meta-pixel it belongs to. Our conversion function $\eta_x$ scales $i$ by the value of the meta-pixel $j$:
%$\eta_x(\hat{z})[i] = \hat{z}[j]x[i].$
%Similarly, for timeseries \slime{} extends the conversion function proposed in~\cite{leftist}.
%The target instance is split into 20 segments of equal size such that each segment relates to one feature. Then, for each time-stamp $t$ in segment $j$, $\eta_x$ scales the value observed at time $t$ by the value of the $j$-th segment:
%$\eta_x(\hat{z})[t] = \hat{z}[j]x[t].$

Finally, for tabular data we consider one surrogate feature per original feature. Therefore, the generating distribution $\nu_\sigma$ is the centered multivariate Gaussian distribution with covariance $\sigma^2\pmb{I}$, and the function $\eta_x(\hat{z}) = x + \hat{z}$.

\begin{remark}
The design of a proper distribution $\nu_\sigma$ and a proper function $\eta_x$ requires the black-box model to handle examples living in a continuous space. As a consequence, \slime{} cannot be defined for text data.
\end{remark}

\section{Experiments}
\label{sec:experiments}

We now show-case the impact of the bandwidth $\sigma$ on the fidelity of LIME and \slime{} explanations. 
%The fidelity is measured based on the model predictions and on the list of features used by the model. The experiments demonstrate that the fidelity is impacted by $\sigma$, and that in most settings (i) the smaller $\sigma$ the better and (ii) by handling smaller values of $\sigma$ \slime{} is more faithful than LIME. 
We first detail our experimental setup and then elaborate on our findings.

\subsection{Experimental Settings}
\paragraph{Datasets and Black Boxes. }
%For the sake of genericity, our study consider several datasets and several black-box models which where chosen for their diversity.
We conduct our experiments on a variety of datasets, comprising Cifar10 \cite{cifar10} and MNIST \cite{mnist} for image data, the FordA and StarlightCurves time series datasets from the \emph{UEA \& UCR Time Series Classification Repository}, and the Compas and Diabetes datasets from the \emph{UCI Machine Learning Repository} for tabular data. We also consider a selection of black-box models, which may be smooth or piece-wise constant, simple or complex, interpretable or not.

\paragraph{Protocol and Metrics.} For each combination of dataset, model, and explanation module, we compute the average value of the experimental metrics for different values of $\sigma$ on the test instances of the dataset.
The experimental metrics were introduced in Section~\ref{sec:preliminaries}: the $R^2$ score for all models, and the precision/recall or the coverage for the interpretable models, i.e., those for which a ground truth is available.
All these metrics take values either in $(-\infty,1]$ or in $[0,1]$, and higher values denote higher fidelity.

\begin{figure*}[t!]
    \captionsetup[subfigure]{aboveskip=-1pt,belowskip=-1pt}
	\centering%
	\begin{subfigure}[b]{0.24\textwidth}
        \includegraphics[width=\linewidth]{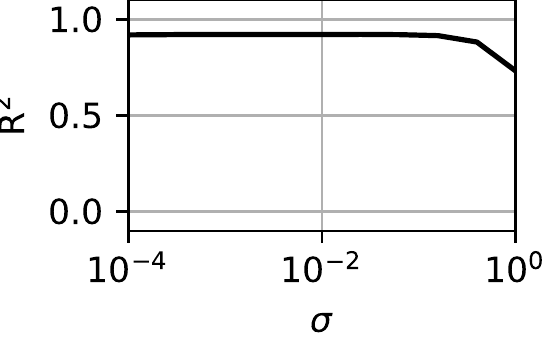}
		\caption{\slime{} on LS}
		\label{fig:s-LIME_LS}
	\end{subfigure} \hfill
	\begin{subfigure}[b]{0.24\textwidth}
        \includegraphics[width=\linewidth]{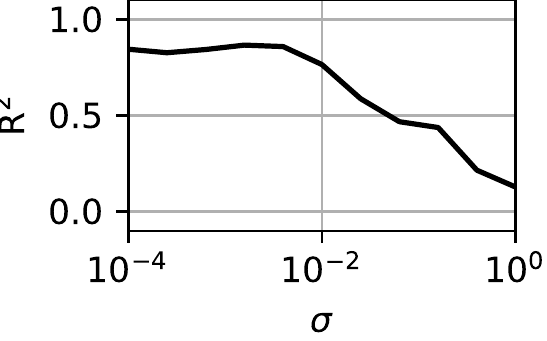}
		\caption{\slime{} on RES.}
		\label{fig:s-LIME_RESNET}
	\end{subfigure} \hfill
	\begin{subfigure}[b]{0.24\textwidth}
        \includegraphics[width=\linewidth]{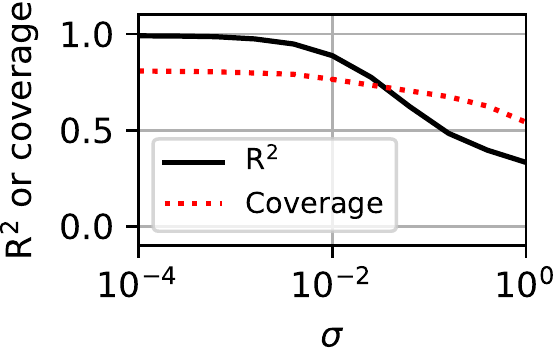}
		\caption{\slime{} on LR}
		\label{fig:s-LIME_LR}
	\end{subfigure} \hfill
	\begin{subfigure}[b]{0.24\textwidth}
        \includegraphics[width=\linewidth]{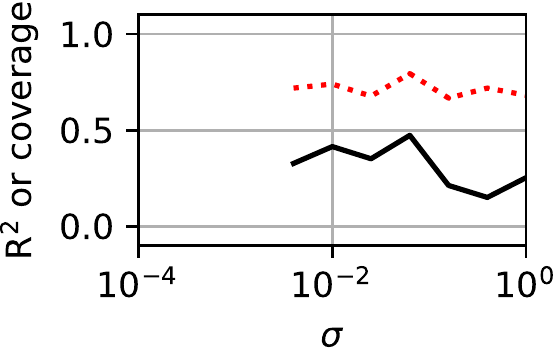}
		\caption{\slime{} on FS}
		\label{fig:s-LIME_FS}
	\end{subfigure}\\\medskip	
	\begin{subfigure}[b]{0.24\textwidth}
        \includegraphics[width=\linewidth]{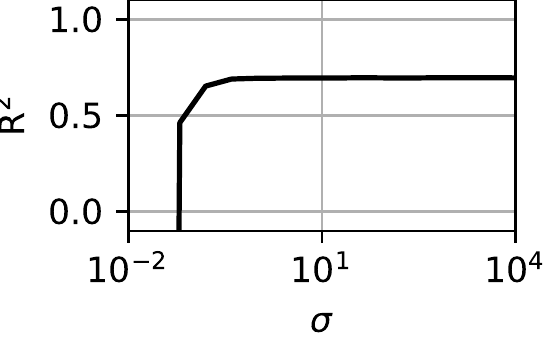}
		\caption{LIME on LS}
		\label{fig:LIME_LS}
	\end{subfigure} \hfill
	\begin{subfigure}[b]{0.24\textwidth}
        \includegraphics[width=\linewidth]{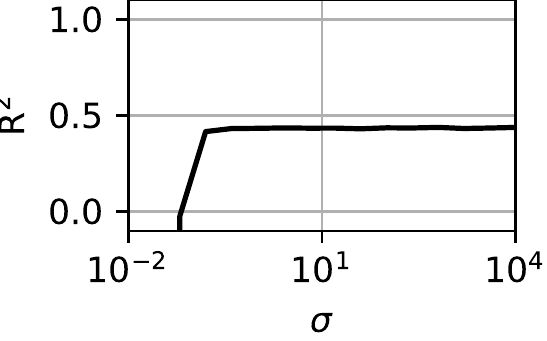}
		\caption{LIME on RESNET}
		\label{fig:LIME_RESNET}
	\end{subfigure} \hfill
	\begin{subfigure}[b]{0.24\textwidth}
        \includegraphics[width=\linewidth]{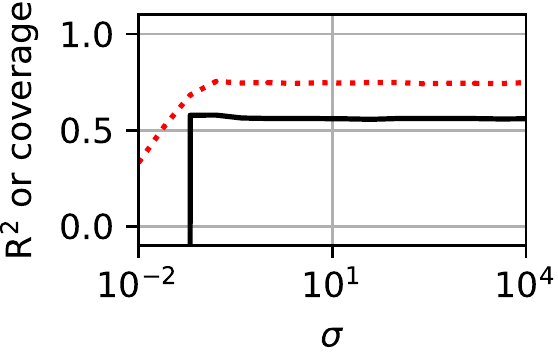}
		\caption{LIME on LR}
		\label{fig:LIME_LR}
	\end{subfigure} \hfill
	\begin{subfigure}[b]{0.24\textwidth}
        \includegraphics[width=\linewidth]{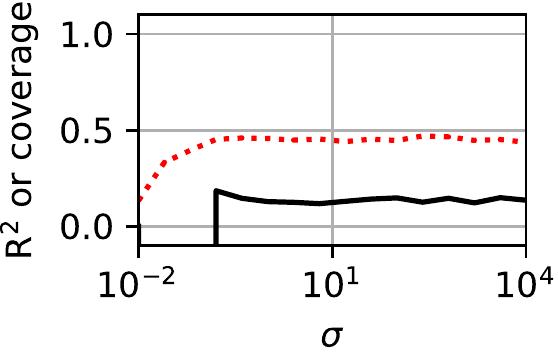}
		\caption{LIME on FS}
		\label{fig:LIME_FS}
	\end{subfigure}	
	\caption{$R^2$ and coverage vs. $\sigma$ on the StarlightCurves dataset. Each subplot corresponds to a couple (explainer, dataset). The plotted results are averaged on the instances of the test dataset. Recall that for \slime{} $\sigma$ is defined in $(0, 1]$.   
	}
	\label{fig:StarlightCurves}
\end{figure*}

\subsection{Impact of $\sigma$}

%Our first analysis aim at highlighting the impact of $\sigma$ on the fidelity of the explanation.
To study the impact of $\sigma$ on the fidelity of the LIME and \slime{} explanations, we plot the surrogate's adherence on the StarlightCurves dataset for several black-box models all using 100 random shapelets as input features. The models include Learning Shapelets (LS) \cite{learningshapelets}, RESNET \cite{resnet}, Fast Shapelets (FS) \cite{fastshapelets}, and a sparse logistic regression (LR, with $L_1$-regularization to enforce at most 10 features). The results are depicted in Figure \ref{fig:StarlightCurves}. We set $k=6$ for the number of features in explanations~\cite{lime}.
%These models are explained with both \slime{} and LIME, requesting explanations using at most $k=6$ surrogate features.

%We focus here on a unique dataset (StarlightCurves) to unable plotting of the evolution of the metrics w.r.t. $\sigma$.
%Figure \ref{fig:StarlightCurves} gathers such plots for several black-box models: LS, RESNET, FS, and LR, which respectively denote Learning Shapelets \cite{learningshapelets}, RESNET \cite{resnet}, Fast Shapelets \cite{fastshapelets}, and a logistic regression learned on 100 features, each feature being associated to a random shapelet.
%LR is penalized with a $L_1$-norm regularization to enforce the use of at most 10 features.
%These models are explained with both \slime{} and LIME, requesting explanations using at most $k=6$ surrogate features.

\begin{figure*}[t!hb]
    \captionsetup[subfigure]{aboveskip=-1pt,belowskip=-1pt}
	\centering%
	\begin{subfigure}[b]{0.4\textwidth}
        \includegraphics[width=\linewidth]{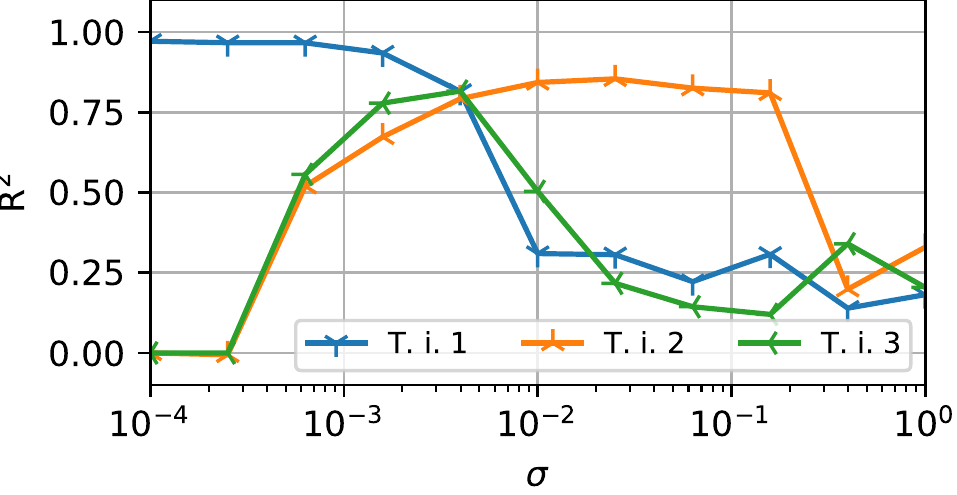}
		\caption{\slime{} on RESNET}
		\label{fig:s-LIME_RESNET-3instances}
	\end{subfigure}	\quad
	\begin{subfigure}[b]{0.4\textwidth}
        \includegraphics[width=\linewidth]{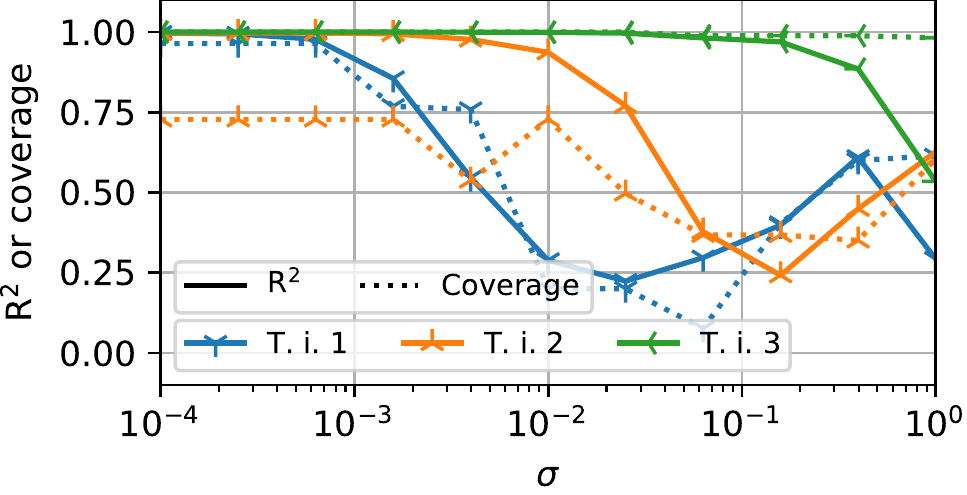}
		\caption{\slime{} on LR}
		\label{fig:s-LIME_LR-3instances}
	\end{subfigure}
	\caption{$R^2$ and coverage vs. $\sigma$ on the StarlightCurves dataset. Each subplot corresponds to a couple (explainer, dataset). Each curve corresponds to one target instance.
	}
	\label{fig:StarlightCurves-3instances}
\end{figure*}

We observe that very local \slime{} neighborhoods lead to higher adherence and coverage, except for FS. This translates into more faithful explanations as $\sigma$ approaches zero, where LIME cannot deliver proper explanations. 
%As a consequence, \slime{} delivers the most faithful explanations as it is able to generate infinitely small neighborhoods while LIME cannot deliver proper neighborhoods for small values of $\sigma$.
%With FS, the best linear explanation is LIME's one, and the best bandwith for \slime{} is about $0.1$.
In contrast, LIME achieves higher adherence and coverage for FS, because this model is a decision tree. Hence, the decision function is piece-wise constant and its gradient is zero almost every-where. When $\sigma$ is small enough, \slime{} recovers this gradient and returns an explanation with null coefficients, which has little practical value. That said, 
%But, while this explanation is faithful, it has a $R^2$ score of 0 and it is useless in term of explanation.
a wider locality can still yield a more informative explanation.

We also remark that, for complex models, the best value for $\sigma$ may depend on the target instance.
This is corroborated by Figure \ref{fig:StarlightCurves-3instances} that shows the disaggregated results for 3 instances on RESNET, a deep neural network. We can observe that the adherence is maximal when $\sigma$ is equal $10^{-4}$, $3\times10^{-3}$, and $2\times10^{-2}$ respectively. 
On the other hand, the same values of $\sigma$ are optimal for all examples on a simpler LR model.

Finally, we highlight that the coverage peaks when the adherence is maximal both at the instance (Figure~\ref{fig:s-LIME_LR-3instances}) and dataset level (Figures~\ref{fig:StarlightCurves}(cdgh)).
%, 
%maximized when $\sigma$ is chose corresponding to the highest $R^2$ score. 
This shows the pertinence of the $R^2$ score as metric to select the right level of locality.

\begin{table}[t]
%\centering
\caption{Best average recall and precision, or coverage (std. in parentheses) on different datasets and interpretable black-box classifiers.
}
\label{tab:recall}
\begin{tabular}{lllcccc}
     \textbf{Data type} & \textbf{Dataset} & \textbf{Model}   & \multicolumn{2}{c}{\textbf{\slime{}}}& \multicolumn{2}{c}{\textbf{LIME}}\\
     \cmidrule(r){4-5} \cmidrule(){6-7}
     &&& Rec. or Cov. & Precision &  Rec. or Cov. & Precision\\
     \toprule
     Timeseries
     %& GunPoint
     %& LR on shapelets  &  \textbf{0.71} (0.19) & - (-) & 0.59 (0.17)  &  - (-) \\
     %&& Fast Shapelets  &  &- (-)  &  & - (-) (Lime meilleur)  \\
     & FordA
     & LR on shapelets  &  \textbf{0.87} (0.15) & - (-) & 0.73 (0.17)  &  - (-)\\
     && Fast Shapelets  & \textbf{0.51} (0.30) & - (-) & \textbf{0.49} (0.27) & - (-) \\
     & Starlight-& LR on shapelets  &  \textbf{0.81} (0.17) & - (-) & 0.75 (0.17)  &  - (-) \\
     &Curves& Fast Shapelets  & \textbf{0.68} (0.19) &- (-)  & 0.45 (0.15) & - (-)  \\
     \midrule
     Tabular data
     & Diabetes
     & Logistic Reg.  &  \textbf{1.00} (0.00) & \textbf{1.00} (0.00) & 0.88 (0.12)  &  0.88 (0.12)\\
     && Dec. Tree  & \textbf{0.95} (0.13) & \textbf{0.81} (0.20)  & \textbf{0.94} (0.14) & \textbf{0.80} (0.20)  \\
     & Compas
     & Logistic Reg.  &  \textbf{1.00} (0.00) & \textbf{1.00} (0.00) & 0.52 (0.21)  &  0.52 (0.21) \\
     && Dec. Tree  & \textbf{0.66} (0.33) & 0.25 (0.00)  & \textbf{0.65} (0.33)  & \textbf{0.33} (0.00)  \\
     %& Wine
     %& Logistic Reg.  &  0.48 (0.15) & 0.68 (0.22) & \textbf{0.53} (0.20)  &  \textbf{0.74} (0.23) \\
     %&& Dec. Tree  & 0.80 (0.19) & 0.75 (0.18)  & \textbf{0.81} (0.19)  & 0.75 (0.18) \\
     \bottomrule
\end{tabular}
\end{table}

\begin{table}[t!hp]
\caption{Best average $R^2$ (std. in parentheses) on different datasets and black-box classifiers.
MLP stands for a neural network with one hidden layer composed of 100 neurons and logistic sigmoid activation function.
Column \emph{Int.} indicates interpretable black-box models~($\checkmark$).
FS, DT and RF are put aside as they are piecewise constant models.
}
\label{tab:r2}
\centering
\begin{tabular}{llccccccc}
     \textbf{Data type}  & \textbf{Model} & \textbf{Int.} & $k$  & \textbf{\slime{}}& \textbf{LIME} & $k$  & \textbf{\slime{}}& \textbf{LIME} \\
     \toprule
     Images &&& \multicolumn{3}{c}{MNIST} & \multicolumn{3}{c}{Cifar10}\\
     \cmidrule(r){4-6}\cmidrule(l){7-9}
     & Alexnet %\cite{alexnet}
            &  & 10& \textbf{0.80} (0.28) & 0.58 (0.20) & 10 & \textbf{0.84} (0.10)  & 0.55 (0.25) \\
     & VGG16 %\cite{vgg}
            &  & 10&  {\textbf{0.56} (0.43)}  & {\textbf{0.57} (0.21)}  & 10& \textbf{0.69} (0.13) & 0.50 (0.27) \\
     %& Fashion MNIST
     %& LR on haar  &  &   \\
     %&& Simple CNN  &  &   \\
     \midrule
     Timeseries &&& \multicolumn{3}{c}{FordA} & \multicolumn{3}{c}{StarlightCurves}\\
     \cmidrule(r){4-6}\cmidrule(l){7-9}
     & Learning Shapelets &  & 6&\textbf{0.84} (0.08) & 0.57 (0.15) & 6&\textbf{0.92} (0.07) & 0.70 (0.07)\\
     & RESNET &  & 6&\textbf{0.73} (0.20) & 0.10 (1.05) & 6&\textbf{0.87} (0.15) & 0.44 (0.15) \\
     & LR on Shapelets  & \checkmark& 6& \textbf{1.00} (0.01) & 0.56 (0.13) & 6&  \textbf{0.99} (0.02) & 0.58 (0.12)\\
     \cmidrule(r){2-3}\cmidrule(r){4-6}\cmidrule(l){7-9}
     & Fast Shapelets  & \checkmark & 6&0.15 (0.18) & \textbf{0.19} (0.14)  & 6& \textbf{0.25} (0.13) & 0.19 (0.16) \\
     %& GunPoint
     %& LS & X &\textbf{0.86} (0.13) & 0.79 (0.1)\\
     %&& RESNET & X &\textbf{0.91} (0.064) & 0.18 (0.093)\\
     %&& LR on shapelets  & \checkmark&  &\\
     %\cmidrule{3-5}
     %&& Fast Shapelets  & \checkmark&  & (Lime meilleur)  \\
     %\cmidrule{3-5}
     \midrule
     Tabular data &&& \multicolumn{3}{c}{Diabetes} & \multicolumn{3}{c}{Compas}\\
     \cmidrule(r){4-6}\cmidrule(l){7-9}
     & Logistic Regression  & \checkmark& 4& \textbf{1.00} (0.00) & \textbf{0.99} (0.01) & 11& \textbf{1.00} (0.00)  & 0.42 (0.23) \\
     & MLP  &  & 4&  \textbf{0.97} (0.03) & 0.72 (0.13) & 6 & \textbf{0.79} (0.01)  & 0.31 (0.16)\\
     \cmidrule(r){2-3}\cmidrule(r){4-6}\cmidrule(l){7-9}
     & Decision Tree & \checkmark & 3& \textbf{0.46} (0.09)  & \textbf{0.46} (0.10) & 3& 0.34 (0.00) & \textbf{0.36} (0.00)  \\
     & Random Forest %\cite{breiman2001randomforest}
        &  & 4& \textbf{0.62} (0.03) & 0.58 (0.12) & 6 & \textbf{0.30} (0.01)  & \textbf{0.30} (0.02) \\
     %\cmidrule{3-5}
     %& Iris 
     %& Random Forests &  & (Lime meilleur)  \\
     %&& MLP (logistic)  &  & (Lime meilleur)  \\   \cmidrule{3-5}
     %&& Logistic Reg.  &  & (Lime meilleur)  \\
     %&& Dec. Tree  &  & (Lime meilleur)  \\ \cmidrule{3-5}
     %& Wine ($k=6$)
     %& Random Forests & 0.64 (0.21) & 0.54 (0.11)  \\
     %&& MLP (logistic)  & 0.98 (0.02)  & 0.59 (0.22) \\
     %\cmidrule{3-5}
     %&& Logistic Reg.  & 0.97 (0.03)  & 0.63 (0.11) \\
     %&& Dec. Tree  & 0.30 (0.13)  & 0.30 (0.13) \\
     \bottomrule      
\end{tabular}
\end{table}

\subsection{Fidelity Analysis}

Tables \ref{tab:recall} and \ref{tab:r2} show the average scores obtained by \slime{} and LIME when $\sigma$ is selected to maximize the aggregated adherence ($R^2$ score) in the test instances of the experimental datasets. Table~\ref{tab:recall} shows recall, precision, and coverage for the interpretable models, whereas Table~\ref{tab:r2} provides the $R^2$ score for all models.
%This study on a larger set of data strengthens the conclusions of the study limited to StarlightCurve dataset.

Firstly, we remark that \slime{}'s explanations are strictly more faithful than LIME's except for piecewise constant models (FS, DT, and RF).
That said, this does not prevent \slime{} from achieving higher adherence for such models on some datasets when we look at a larger vicinity.

Secondly, the $R^2$ score is a good proxy to predict the best neighborhood in terms of recall, precision, or coverage.
This is a strong result from an application point of view. Practitioners are mostly interested by the features that are actually used by the black-box model. %and local linear explanations are tools to identify those. 
%and they identify these features based on the local linear explanation.
%In other world, practitioners hope for high values of recall, precision, or coverage.
For cases where those actual features are unknown,
%Our results demonstrate that, while these metrics cannot be measured, 
the $R^2$ score enables the computation of faithful linear explanations that can identify the important features.

\section{Related Work}
\label{sec:related-work-julien}
%An important body of literature has studied the impact of the different components and parameters of LIME on the quality of the resulting explanations. This has lead to multiple extensions of the original approach that we survey in this section. 
\paragraph{Feature-attribution explanations.} Methods such as DeepLIFT~\cite{deeplift}, Integrated Gradients (IG)~\cite{ig}, SHAP~\cite{shap}, or LIME~\cite{lime} compute importance local attribution scores for the features of a black-box ML model. Among those, SHAP and LIME are model-agnostic and compute linear surrogates learned from artificial neighbors. Despite these similarities, the semantics of their explanations are different as confirmed by existing studies~\cite{leaf}. While LIME approximates -- often coarsely -- the instantaneous gradient of the black box w.r.t. the input features~\cite{explaining-lime}, SHAP computes -- or rather approximates -- the Shapley values~\cite{shap}, which quantify the feature contributions to the difference between the model's answer on a baseline instance and the target. The baseline depends on the use case, e.g., a single-color image (represented by the vector $\mathbb{0}^{\hat{d}}$ in the surrogate space). This makes SHAP and LIME complementary methods rather than competitors. 
%In~\cite{shap}, it is shown that we can compute the Shapley values by generating artificial instances as in LIME, dropping the regularization, and concentrating the training weights on the closest and farthest instances in the neighborhood $\hat{\mathcal{X}}$. 
%SHAP's explanations offer interesting theoretical guarantees such as local accuracy, i.e., the surrogate is always accurate on the target instance. 

\paragraph{LIME Extensions.} 
An important body of literature has studied the impact of the different components and parameters of LIME on the quality of the explanations. This has led to multiple extensions of the original LIME algorithm.
As opposed to this work, some extensions~\cite{ALIME,DLIME,OptiLIME} tackle the instability of LIME, i.e., the fact that two executions of the algorithm with the same input may not deliver the same explanation. This instability originates from the randomness in the different steps of the approach, e.g., sampling in the surrogate space, non-deterministic conversion functions, etc. On those grounds, the techniques to tackle instability are diverse. ALIME~\cite{ALIME}, for example, resorts to a denoising auto-encoder to create a surrogate space that characterizes the data manifold more accurately. 
%By doing so, ALIME can learn the surrogate on realistic neighbor instances boosting both fidelity -- defined by the $R^2$ score -- and instability.
DLIME~\cite{DLIME}, in contrast, applies hierarchical agglomerative clustering on the training instances to identify the closest neighbors of the target and use them to learn the surrogate. 
In another line of thought, the authors of OptiLIME~\cite{OptiLIME} study the relationship between the bandwidth $\sigma$, the adherence, and the instability of LIME. Similar to our work, the authors highlight the importance of choosing the right $\sigma$ in a per-instance basis.  Moreover, they show an inverse relationship between $\sigma$ and explanation instability. 
This observation constitutes the basis of a method to select the bandwidth $\sigma$ that yields the best trade-off between adherence and instability. 
We highlight that all these approaches have been proposed only for tabular data, and that none of them takes into account recall, precision, or coverage fidelity.   

Other extensions of LIME have focused entirely on improving fidelity. ILIME~\cite{iLIME} proposes the use of influence functions in order to up-weight the neighbors that play a higher role in the linear fit of the surrogate. 
%We believe that such a technique could be applied in combination with s-LIME.
QLIME-A~\cite{QLIME} proposes to extend the local surrogate to report quadratic relationships for cases where a linear surrogate is still inaccurate. While quadratic functions do exhibit higher fit capabilities, their interpretability in general settings is debatable.

\section{Conclusion}
\label{sec:conclusion}
In this paper we have introduced \slime{}, an extension of LIME that reconciles locality and fidelity for linear explanations. We argue that LIME can produce degenerated explanations as locality -- controlled through the bandwidth $\sigma$ -- increases. We solve this paradox by means of a new neighbor generation process on a continuous surrogate space. Our experiments on image, time series, and tabular data suggest that this strategy can provide even more faithful linear explanations with gradient-compliant semantics that are not affected by high locality. As a future work, we envision to investigate the fidelity of $\slime{}$ explanations with other generating distributions and conversion functions, as well as to study the impact on the stability of the explanations.

\subsubsection{Acknowledgements}
%Please place your acknowledgments at the end of the paper, preceded by an unnumbered run-in heading (i.e. 3rd-level heading).
This research was partially supported by
the Inria Project Lab “Hybrid Approaches for Interpretable AI” (HyAIAI), the project “Framework for Automatic Interpretability in Machine Learning” financed by the French National Research Agency (ANR JCJC FAbLe), and the network on the foundations of trustworthy AI, integrating learning, optimisation, and reasoning (TAILOR) financed by the EU’s Horizon 2020 research and innovation program under agreement 952215.

\bibliographystyle{splncs04}
\bibliography{references}
\end{document}